\documentclass[sigconf]{aamas} 

\usepackage{balance} % for balancing columns on the final page

\usepackage{booktabs}
\usepackage{algorithm}
\usepackage[noend]{algpseudocode}
\usepackage{graphicx}
\usepackage{subfig}
\usepackage{hyperref}       % hyperlinks

\usepackage{url}            % simple URL typesetting

\usepackage{booktabs}       % professional-quality tables

\usepackage{amsfonts}       % blackboard math symbols

\usepackage{amsmath}       
\usepackage{amsthm}       

\usepackage{nicefrac}       % compact symbols for 1/2, etc.

\usepackage{microtype}      % microtypography

\usepackage{tikz}

\usepackage{graphicx}

\usepackage{verbatim}
\usepackage{multirow}

\usepackage{dblfloatfix}
\setcopyright{none}
\newcommand\blfootnote[1]{%
  \begingroup
  \renewcommand\thefootnote{}\footnote{#1}%
  \addtocounter{footnote}{-1}%
  \endgroup
}

\title{An Analysis of Frame-skipping in Reinforcement Learning}

\author{Shivaram  Kalyanakrishnan}
\affiliation{
  \institution{IIT Bombay}}
\email{shivaram@cse.iitb.ac.in}

\author{Siddharth Aravindan *}
\affiliation{
  \institution{National University of Singapore}}
\email{siddharth.aravindan@comp.nus.edu.sg}

\author{Vishwajeet  Bagdawat *}
\affiliation{
  \institution{Panasonic Corporation}}
\email{vishwajeet.singh@jp.panasonic.com}

\author{Varun Bhatt *}
\affiliation{
  \institution{University of Alberta}}
\email{vbhatt@ualberta.ca}

\author{Harshith  Goka *}
\affiliation{
  \institution{Samsung Research}}
\email{h9399.goka@samsung.com}

\author{Archit Gupta *}
\affiliation{
  \institution{Rubrik}}
\email{archit.gupta@rubrik.com}

\author{Kalpesh Krishna *}
\affiliation{
  \institution{UMass Amherst}}
\email{kalpesh@cs.umass.edu}

\author{Vihari Piratla *}
\affiliation{
  \institution{IIT Bombay}}
\email{viharipiratla@gmail.com}

\begin{abstract}
In the practice of sequential decision making, agents are often designed to sense state at regular intervals of $d$ time steps, $d > 1$, ignoring state information in between sensing steps. While it is clear that this practice can reduce sensing and compute costs, recent results indicate a further benefit. On many Atari console games, reinforcement learning (RL) algorithms deliver substantially better policies when run with $d > 1$---in fact with $d$ even as high as $180$. In this paper, we investigate the role of the parameter $d$ in RL; $d$ is called  the ``frame-skip'' parameter, since states in the Atari domain are images. For evaluating a fixed policy, we observe that under standard conditions, frame-skipping does not affect asymptotic consistency. Depending on other  parameters, it can possibly even benefit learning. To use $d > 1$ in the control setting, one must first specify which $d$-step open-loop action sequences can be executed in between sensing steps. We focus on ``action- repetition'', 
the common restriction of this choice to $d$-length sequences of the same action. We define a task-dependent quantity called the ``price of inertia'', in terms of which we upper-bound the loss incurred by action-repetition. We show that this loss may be offset by the gain brought to learning by a smaller task horizon. Our analysis is supported by experiments on different tasks and learning algorithms. 

\blfootnote{* Equal contribution}

\thanks{Equal contribution}
\end{abstract}


\DeclareMathOperator*{\argmax}{argmax}
\DeclareMathOperator*{\argmin}{argmin}
\DeclareMathOperator*{\eqdef}{\stackrel{\text{\tiny def}}{=}}

\newtheorem{theorem}{Theorem}

\newtheorem{proposition}[theorem]{Proposition}
\newtheorem{lemma}[theorem]{Lemma}

\begin{document}

%%% The following commands remove the headers in your paper. For final 
%%% papers, these will be inserted during the pagination process.

\pagestyle{fancy}
\fancyhead{}

%%% The next command prints the information defined in the preamble.

\maketitle

\section{Introduction}
\label{sec:introduction}
Sequential decision making tasks are most commonly formulated as Markov Decision Problems~\cite{Puterman:1994}. An MDP models a world with state transitions that depend on the action an agent may choose. Transitions also yield rewards. Every MDP is guaranteed to have an optimal \textit{policy}: a state-to-action mapping 
that maximises expected long-term reward~\cite{Bellman:1957}. Yet, on a given task, it might not be \textit{necessary} to sense state at each time step in order to optimise performance. For example, even if the hardware allows a car
to sense state and select actions every millisecond, it might suffice on typical roads to do so once every ten milliseconds. The reduction in reaction time by so doing might have a negligible effect on performance, and be justified by the substantial savings in  sensing and computation.

Recent empirical studies 
bring to light a less obvious benefit from reducing the frequency of sensing: sheer 
improvements in performance when behaviour is \textit{learned}~\cite{Braylan+HMM:2015,durugkar2016deep,Lakshminarayanan+SR:2017}. On the popular Atari console games benchmark~\cite{Bellemare+NVB:2014}  for reinforcement learning (RL), reduced sensing takes the form of ``frame-skipping'', since agents in this domain sense image frames and respond with actions. In the original implementation, sensing is limited to every $4$-th frame, with the intent of lightening the computational load~\cite{mnih2015human}. However, subsequent research has shown that higher performance levels can be reached by skipping up to $180$ frames in some games~\cite{Braylan+HMM:2015}.

We continue to use the term ``frame-skipping'' generically across all sequential decision making tasks, denoting by parameter $d \geq 1$ the number of time steps between sensing steps (so $d = 1$ means no frame-skipping). For using $d > 1$, observe that it is necessary to specify an entire sequence of actions, to execute in an open-loop fashion, in between sensed frames. The most common strategy for so doing is ``action-repetition'', whereby the same atomic action is repeated $d$ times. Action-repetition has been the default strategy for implementing frame-skipping on the Atari console games, both when $d$ is treated as a hyperparameter~\cite{mnih2015human,Braylan+HMM:2015} and when it is adapted on-line, during the agent's lifetime~\cite{durugkar2016deep,Lakshminarayanan+SR:2017,Sharma+LR:2017}.

In this paper, we analyse the role of frame-skipping and action-repetition in RL---in short, examining \textit{why} they work. We begin by surveying topics in sequential decision making that share connections with frame-skipping and action-repetition (Section~\ref{sec:literaturesurvey}). Thereafter we provide formal problem definitions in Section~\ref{sec:problemdefinition}. In Section~\ref{sec:predictionwithframe-skipping}, we take up the problem of \textit{prediction}: estimating the value function of a fixed policy. We show that prediction with frame-skipping continues to give consistent estimates when used with linear function approximation. Additionally, $d$ serves as a handle to simultaneously tune the amount of bootstrapping and the task horizon. In Section~\ref{sec:controlwithaction-repetition}, we investigate the \textit{control} setting, wherein behaviour is adapted based on experience. First we define a task-specific quantity called the ``price of inertia'', in terms of which we bound the loss incurred by action-repetition. Thereafter we show that frame-skipping might still be beneficial in aggregate because it reduces the effective task horizon. In Section~\ref{sec:empiricalevaluation}, we augment our analysis with empirical findings on different tasks and learning algorithms. Among our results is a successful demonstration of learning \textit{defensive} play in soccer, 
a hitherto less-explored side of the game~\cite{ALA16-hausknecht}. We conclude with a summary in Section~\ref{sec:conclusion}.

\section{Literature Survey}
\label{sec:literaturesurvey}

Frame-skipping may be viewed as an instance of (partial) \textit{open-loop control}, under which a predetermined sequence of (possibly different) actions is executed without heed to intermediate states. Aiming to minimise sensing, Hansen \textit{et al.}~\cite{hansen1997reinforcement} propose a framework for incorporating variable-length open-loop action sequences in regular (closed-loop) control. The primary challenge in general open-loop control is that the number of action sequences of some given length $d$ is exponential in $d$. Consequently, the main focus in the area is on strategies to prune corresponding data structures~\cite{Tan:1991,McCallum:1996-thesis,hansen1997reinforcement}. Since action repetition restricts itself to a set of actions with size \textit{linear} in $d$, it allows for $d$ itself to be set much higher in practice~\cite{Braylan+HMM:2015}. 

%In Section~\ref{sec:inducedmdpsandanalysis}, we show that on finite MDPs, AR can handle even infinite $d$ with a \textit{finite} memory footprint---a theoretical result that appears to establish a clear divide between AR and general open-loop control.

To the best of our knowledge, the earliest treatment of action-repetition in the form we consider here is by Buckland and Lawrence \cite{buckland1994transition}. While designing agents to negotiate a race track, these authors note that successful controllers need only change actions at ``transition points'' such as curves, while repeating the same action for long stretches. They propose an algorithmic framework that only keeps transition points in memory, thereby achieving savings. In spite of limitations such as the assumption of a discrete state space, their work provides conceptual guidance for later work. For example, Buckland ~\cite[see Section 3.1.4]{Buckland:1994-thesis} informally discusses ``inertia'' as a property determining a task's suitability to action repetition---and we formalise the very same concept in Section~\ref{sec:controlwithaction-repetition}.

Investigations into the effect of action repetition on \textit{learning} begin with the work of McGovern \textit{et al.}~\cite{McGovern+SF:1997}, who identify two qualitative benefits: improved exploration (also affirmed by Randl{\o}v~\cite{randlov1999learning}) and the shorter resulting task horizon. While these inferences are drawn from experiments on small, discrete, tasks, they find support in a recent line of experiments on the Atari environment, in which neural networks are used as the underlying representation~\cite{durugkar2016deep}. In the original implementation of the DQN algorithm on Atari console games, actions are repeated $4$ times, mainly to reduce computational load~\cite{mnih2015human}. However, subsequent research has shown that higher performance levels can be reached by persisting actions for longer---up to 180 frames in some games~\cite{Braylan+HMM:2015}. More recently, Lakshminarayanan \textit{et al.}~\cite{Lakshminarayanan+SR:2017} propose a policy gradient approach to optimise $d$ (fixed to be either 4 or 20) on-line. Their work sets up the FiGAR (Fine Grained Action Repetition) algorithm~\cite{Sharma+LR:2017}, which optimises over a wider range of $d$ and achieves significant improvements in many Atari games. It is all this empirical evidence in favour of action repetition that motivates our quest for a theoretical understanding.

The idea that ``similar'' states will have a common optimal action also forms the basis for \textit{state aggregation}, a process under which such states are grouped together~\cite{Li+WL:2006}. In practice, state aggregation usually requires domain knowledge, which restricts it to low-dimensional settings~\cite{silva2009compulsory,dazeleycoarse}. Generic, theoretically-grounded state-aggregation methods are even less practicable, and often validated on tasks with only a handful of states~\cite{Abel+ALL:2018}. In contrast, action-repetition applies the principle that states that occur close in \textit{time} are likely to be similar. Naturally this rule of thumb is an approximation, but one that remains applicable in higher-dimensional settings (the HFO defense task in Section~\ref{sec:empiricalevaluation} has 16 features).

Even on tasks that favour action-repetition, it could be beneficial to
explicitly reason about the ``intention'' of an action, such as to reach a particular state. This type of \textit{temporal abstraction} is formalised as an \textit{option}~\cite{sutton1999between}, which is a closed-loop policy with initial and terminating constraints. As numerous experiments show, action-repetition performs well on a variety of tasks in spite of being open-loop in between decisions. We expect options to be more effective when the task at hand requires explicit skill discovery~\cite{Konidaris:2016}.

In this paper, we take the frame-skip parameter $d$ as discrete, fixed, and known to the agent. Thus frame-skipping differs from Semi-Markov Decision Problems~\cite{Bradtke+Duff:1995}, in which the duration of actions can be continuous, random, and unknown. In the specific context of temporal difference learning, frame-skipping may both be interpreted as a technique to control bootstrapping~\cite{Sutton+Barto:2018}[see Section 6.2] and one to reduce the task horizon~\cite{petrik2009biasing}.

\section{Problem Definition}
\label{sec:problemdefinition}

We begin with background on MDPs, and thereafter formalise the prediction and control problems with frame-skipping.

\subsection{Background: MDPs}
\label{subsec:background:mdps}

A Markov Decision Problem (MDP) $M = (S, A, R, T, \gamma)$
comprises a set of states $S$ and a set of actions $A$. Taking action $a \in A$ from  state $s \in S$ yields a numeric reward with expected value $R(s, a)$, which is bounded in $[-R_{\max}, R_{\max}]$ for some $R_{\max} > 0$. $R$ is the reward function of $M$. The transition function $T$
specifies a probability distribution over $S$: for each
$s^{\prime} \in S$, $T(s, a, s^{\prime})$ is the probability of reaching $s^{\prime}$ by taking action $a$ from $s$. An agent is assumed to interact with $M$ over time, starting at some state. At each time step the agent must decide
which action to take. The action yields a next state drawn stochastically according to $T$ and a reward according to $R$, resulting in a state-action-reward sequence $s_{0}, a_{0}, r_{0}, s_{1}, a_{1}, r_{1}, s_{2}, \dots$. The natural objective of the agent is to maximise some notion of expected long term reward, which we take here to be $\mathbb{E}[r_{0} + \gamma r_{1} + \gamma^{2} r_{2} + \dots],$ where $\gamma \in [0, 1]$ is a discount factor. 
We assume $\gamma < 1$ unless the task encoded by $M$ is \textit{episodic}: that is, all policies eventually reach a terminal state with probability $1$.

A policy $\pi: S \times A \to [0, 1]$, specifies for each $s \in S$, a probability $\pi(s, a)$ of taking action $a \in A$ (hence $\sum_{a \in A} \pi(s, a) = 1$). If an agent takes actions according to such a policy $\pi$ (by our definition, $\pi$ is Markovian and stationary), the expected long-term reward accrued starting at state $s \in S$ is denoted $V^{\pi}(s)$; $V^{\pi}$ is the \textit{value function} of $\pi$. Let $\Pi$ be the set of all policies. It is a well-known result that for every MDP $M$, there is an \textit{optimal policy} $\pi^{\star} \in \Pi$ such that for all $s \in S$ and $\pi \in \Pi$, $V^{\pi^{\star}}(s) \geq V^{\pi}(s)$~\cite{Bellman:1957} (indeed there is always a deterministic policy that satisfies optimality).

In the reinforcement learning (RL) setting, an agent interacts with an MDP by sensing state and receiving rewards, in turn specifying actions to influence its future course. In the \textit{prediction} setting, the agent follows a fixed policy $\pi$, and is asked to estimate the value function $V^{\pi}$. Hence, for prediction, it suffices to view the agent as interacting with a Markov Reward Process (MRP) (an MDP with decisions fixed by $\pi$). In the \textit{control} setting, the agent is tasked with improving its performance over time based on the feedback received. On finite MDPs,
exact prediction and optimal control can both be achieved
in the limit of infinite experience~\cite{Watkins+Dayan:1992,rummery1994line}.

\subsection{Frame-skipping}
\label{subsec:frame-skipping}

In this paper, we consider  generalisations of both prediction and control in which a frame-skip parameter $d \geq 1$ is provided as input in addition to MDP $M$. With frame-skipping, the agent is only allowed to sense every $d$-th state: that is, if the agent has sensed state $s_{t}$ at time step $t \geq 0$, it is oblivious to states $s_{t + 1}, s_{t + 2}, \dots, s_{t + d - 1}$, and next only observes $s_{t + d}$. We assume, however, that the discounted sum of the rewards accrued in between (or the $d$-step \textit{return}), is available to the agent at time step $t + d$. 
Indeed in many applications (see, for example, Section~\ref{sec:empiricalevaluation}), this return $G_{t}^{t+d}$, defined below, can be obtained without explicit sensing of intermediate states. $$G_{t}^{t+d} \eqdef r_{t} + \gamma r_{t + 1} + \dots + \gamma^{d - 1} r_{t + d - 1}.$$

In the problems we formalise below, taking $d = 1$ gives the versions with no frame-skipping.

\noindent\textbf{Prediction problem.} In the case of prediction, we assume that a fixed policy $\pi$ is independently being executed on $M$: that is, for $t \geq 0$, $a_{t} \sim \pi(s_{t}, \cdot)$. However, since the agent's sensing is limited to every $d$-th transition, its
interaction with the resulting MRP becomes a sequence of the form $s_{0}, G_{0}^{d}, s_{d}, G_{d}^{2d}, s_{2d}, G_{2d}^{3d}, \dots$. The agent must estimate $V^{\pi}$ based on this sequence.\\

\noindent\textbf{Control problem.} In the control setting, the agent is itself in charge of action selection. However, due to the constraint on sensing, the agent cannot select  actions based on state at all time steps. Rather, at each time step $t$ that state is sensed, the agent can specify a $d$-length action sequence $b \in A^{d}$, which will be executed open-loop for $d$ steps (until the next sensing step $t + d$). Hence, the agent-environment interaction takes the form $s_{0}, b_{0}, G_{0}^{d}, s_{d}, b_{d}, G_{d}^{2d}, s_{2d}, \dots$, where for $i \geq 0$, $b_{di}$ is a $d$-length action sequence. The agent's aim is still to maximise its long-term reward, but observe that
for $d > 1$, it might not be possible to match $\pi^{\star}$, which is fully closed-loop.\\

In the next section, we analyse the prediction setting with frame-skipping; in Section~\ref{sec:controlwithaction-repetition} we consider the control setting.

\section{Prediction with Frame-skipping}
\label{sec:predictionwithframe-skipping}

In this section, we drop the reference to MDP $M$ and policy $\pi$, only assuming that together they fix an MRP $P = (S, R, T, \gamma)$. For $s , s^{\prime} \in S$, $R(s)$ is the reward obtained on exiting $s$ and $T(s, s^{\prime})$ the probability of reaching $s^{\prime}$. For the convergence of any learning algorithm to the value function $V: S \to \mathbb{R}$ of $P$, it is necessary that $P$ be \textit{irreducible}, ensuring that each state will be visited infinitely often in the limit. If using frame-skip $d > 1$, we must also assume that $P$ is \textit{aperiodic}---otherwise some state  might only be visited in between sensing steps, thus precluding convergence to its value. We proceed with the assumption that $P$ is irreducible and aperiodic---in other words, \textit{ergodic}. Let $\mu: S \to (0, 1)$, subject to $\sum_{s \in S} \mu(s) = 1$, be the stationary distribution on $S$ induced by $P$.

\subsection{Consistency of Frame-skipping}
\label{subsec:consistencyofframeskipping}

If using  frame-skipping with parameter $d \geq 1$, it is immediate that the agent's interaction may be viewed as a regular one (with no frame-skipping) with induced MRP $P_{d} = (S, R_{d}, T_{d}, \gamma^{d})$, in which, if we treat reward functions as $|S|$-length vectors and
transition functions as $|S| \times |S|$ matrices, $$R_{d} = R + \gamma TR + \gamma^{2} T^{2} R + \dots + \gamma^{d - 1} T^{d - 1} R, \text{ and } T_{d} = T^{d}.$$ Since $P$ is ergodic, it follows that $P_{d}$ is ergodic. Thus, any standard prediction algorithm (for example, TD($\lambda$)~\cite[see Chapter 12]{Sutton+Barto:2018}) can be applied on $P$ with frame-skip $d$---equivalent to being applied on $P_{d}$ with no frame-skip---to converge to its value function $V_{d}: S \to \mathbb{R}$. It is easy to see that $V_{d} = V$. Surprisingly, it also emerges that the stationary distribution on $S$ induced by $P_{d}$---denote it $\mu_{d}: S \to (0, 1)$, where $\sum_{s \in S} \mu_{d}(s) = 1$---is identical to $\mu$, the stationary distribution induced by $P$. The following proposition formally establishes the consistency of frame-skipping.

\begin{proposition}
\label{prop:valdistidentical}
For $d \geq 1$, $V_{d} = V$ and
$\mu_{d} = \mu$.
\end{proposition}
\begin{proof}
For the first part, we have that for $s \in S$,
\begin{align*}
V_{d}(s) &= \sum_{i = 0}^{\infty} \gamma^{di} \mathbb{E}[ G_{di}^{di + d} | s_{0} = s] = \sum_{i = 0}^{\infty} \sum_{j = 0}^{d - 1}
\gamma^{di + j}   \mathbb{E}[\  r_{di + j}  | s_{0} = s] \\
&= \sum_{t = 0}^{\infty} \mathbb{E}[\gamma^{t} r_{t}  | s_{0} = s]
= V(s).
\end{align*}
For the second part, observe that since $\mu$ is the stationary distribution induced by $P$, it satisfies $T \mu = \mu$. With frame-skip $d$, we have $T_{d} \mu = T^{d} \mu = T^{d - 1} (T \mu) = T^{d - 1} \mu = \dots = T \mu = \mu,$
establishing that $\mu_{d} = \mu$ (its uniqueness following from the ergodicity of $P_{d}$).
\end{proof}
Preserving the stationary distribution is especially relevant for prediction with approximate architectures, as we see next. 

\subsection{Frame-skipping with a Linear Architecture}
\label{subsec:Frame-skipping with alineararchitecture}

As a concrete illustration, we consider the effect of  frame-skip $d$ in
Linear TD($\lambda$)~\cite[see Chapter 12]{Sutton+Barto:2018}, the well-known family of on-line prediction algorithms. We denote our generalised version of the algorithm TD$_{d}$($\lambda$), where $d \geq 1$ is the given frame-skip parameter and $\lambda \in [0, 1]$ controls  bootstrapping. With a linear architecture, $V_{d}(s)$ is approximated by $w \cdot \phi(s)$, where for $s \in S$, $\phi(s)$ is a $k$-length vector of features. The $k$-length coefficient vector $w$ is updated based on experience, keeping a $k$-length eligibility trace vector for backing up rewards to previously-visited states. Starting with $e_{0} = 0$ and arbitrary $w_{0}$, an update is made as follows for each $i \geq 0$, based on the tuple $(s_{di}, G_{di}^{di + d}, s_{di + d})$:
\begin{align*}
&\delta_{i} \gets G_{di}^{di + d} + \gamma^{d} w_{i} \cdot \phi (s_{di + d}) - 
w_{i} \cdot \phi (s_{di}); \\ %\phantom{aa}
&w_{i + 1} \gets w_{i} + \alpha \delta_{i} e_{i};  \phantom{aa}
e_{i + 1} \gets \gamma^{d} \lambda e_{i} + \phi(s_{di}),
\end{align*}
where $\alpha > 0$ is the learning rate. Observe that with full bootstrapping ($\lambda = 0$), each update by TD$_{d}$($\lambda$) is identical to a multi-step (here $d$-step) backup~\cite[see Chapter 7]{Sutton+Barto:2018} on $P$. The primary  difference, however, is that regular multi-step (and $\lambda$-) backups are performed at every time step. By contrast, TD$_{d}(\lambda)$ makes an update only once every $d$ steps, hence reducing sensing (as well as the computational cost of updating $w$) by a factor of $d$.

\begin{figure*}[b]
\centering
\mbox{
\subfloat[]{
\begin{tikzpicture}[node distance=2cm,auto,scale=0.72,transform shape,->]
\large
    \tikzset{state/.style={circle,draw=black,minimum size=10mm,thick}}
    \node[state] at (0, 0) (s1){$0$};
    \node[state] at (0, 1.5) (s2){$1$};
    \node[state] at (0, 5.0) (sd){$99$};
    \draw(s1) edge [thick, loop left] node {}  (s1);
    \draw(s1) edge [thick] node[below] {}  (s2);
    \draw(s2) edge [thick, loop left] node {}  (s2);
    \draw(s2) edge [thick] node[below] {}  (0, 2.5);
    \node[] at (0, 3.25) {$\Large{\vdots}$};
    \draw(0, 4.0) edge [thick] node[below] {}  (sd);
    \draw(sd) edge [thick, loop left] node {}  (sd);
    \draw(sd) edge [thick, bend left=45] node[right] {}  (s1);
    \draw [draw=black, opacity=0] (-1.8,-1.5) rectangle (1.8,5.7);
\normalsize
  \end{tikzpicture}
\label{fig:mrp}
}
\hspace{-1.0cm}
\subfloat[]{
\includegraphics[width=0.42\textwidth,clip=true]{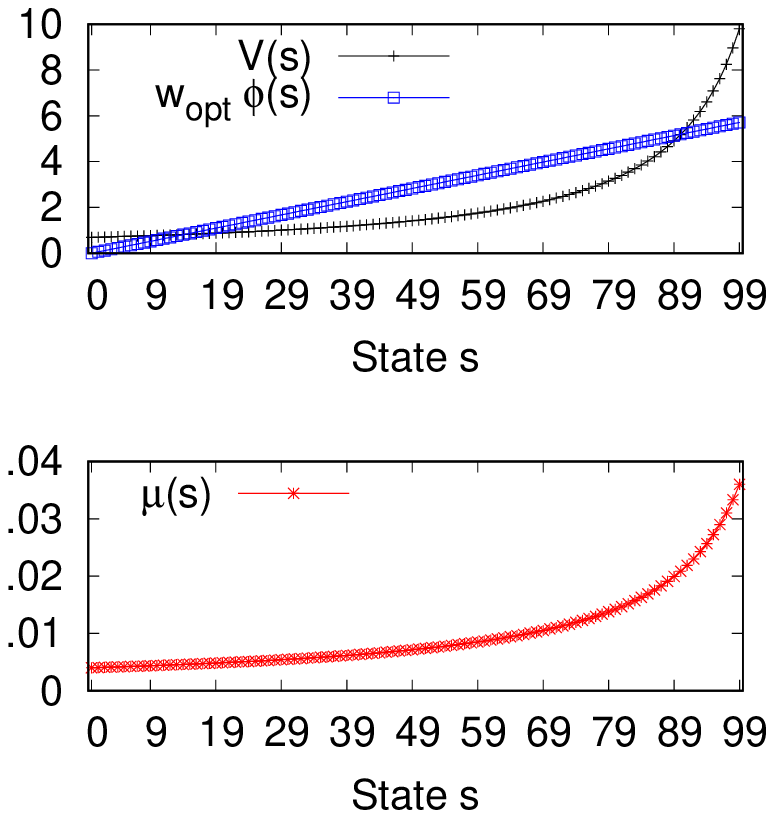}
\label{fig:distandval}
}
\hspace{-1.0cm}

\subfloat[]{
\includegraphics[width=0.45\textwidth,clip=true]{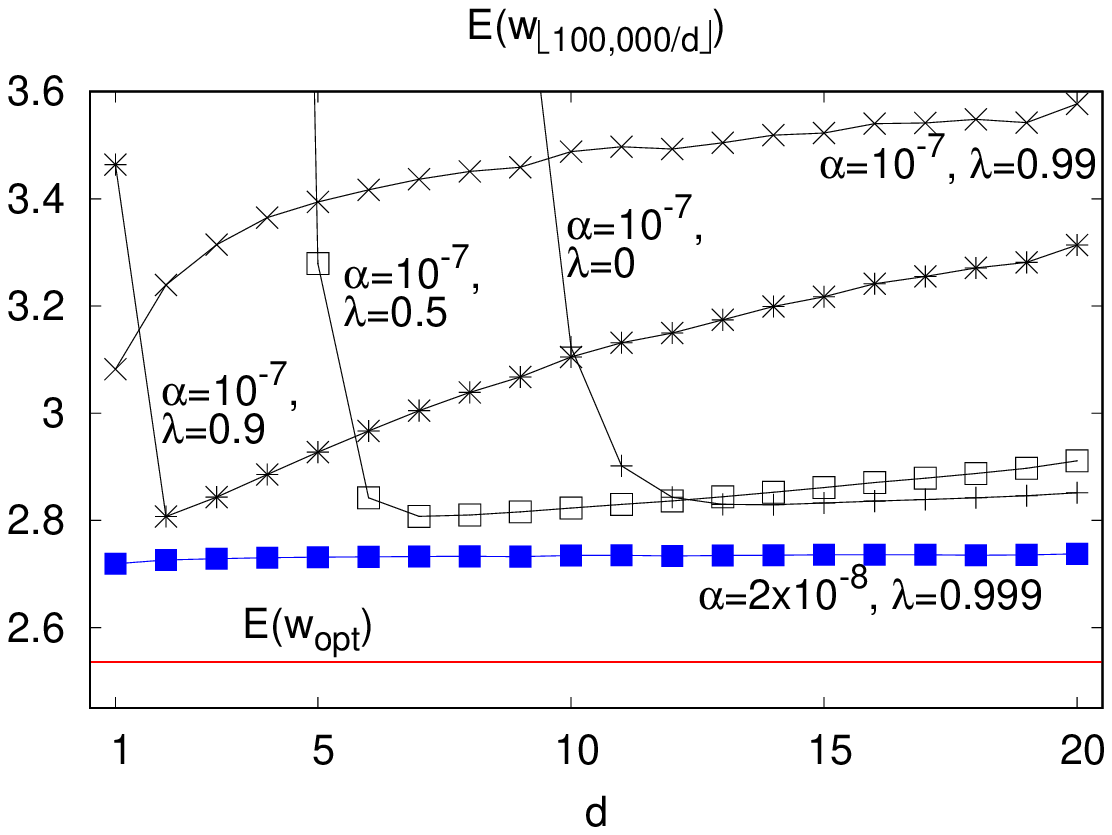}
\label{fig:mrperrors}
}
}
\caption{(a) An MRP with a set of states $S = \{0, 1, \dots, 99\}$. From state $i \in S$, there is a transition with probability $0.9 - 0.8 i / 99$ to state $(i + 1) \mod 99$; otherwise the agent stays in state $i$. All rewards are $0$, except for a $1$-reward when starting from state $99$. The MRP uses a discount factor of $0.99$. (b) The lower panel shows $\mu$; the upper panel shows $V$ as well as its best linear approximation. The linear architecture has a single parameter $w$: for state $i \in S$, $V(i)$ is approximated as $wi$. (c) Value estimation error as a function of $d$ after $100,000$ steps of TD$_{d}$($\lambda$), with $w_{0}$ picked uniformly at random from $[-1, 1]$. Each plot specifies the values of $\alpha$ and $\lambda$; the optimal error is also shown.
Since TD$_{d}$($\lambda$) makes only $(1/d)$ the number of updates of TD($\lambda$), we compensate by running it with learning rate $\alpha d$. Results are averages from 1000 random runs; standard errors are negligible.}
\label{fig:mrp}
\end{figure*}

With linear function approximation, the best result one can hope to achieve is convergence to $$w_{\text{opt}} = \argmin_{w \in \mathbb{R}^{k}} E(w), \text{ where } E(w) = \sum_{s \in S} \mu(s) \{V(s) - w \cdot \phi(s)\}^{2}.$$ It is also well-known that linear $TD(\lambda)$ converges to some $w \in \mathbb{R}^{k}$ such that $E(w) \leq \frac{1 - \gamma \lambda}{1 -  \gamma} E(w_{\text{opt}})$
~\cite{Tsitsiklis+VanRoy:1997}. Note that  TD$_{d}$($\lambda$) on $P$ is the same as TD($\lambda$) on $P_{d}$. Hence, from Proposition~\ref{prop:valdistidentical}, we conclude that TD$_{d}$($\lambda$) on $P$ converges to some $w \in \mathbb{R}^{k}$ such that $E(w) \leq \frac{1 - \gamma^{d} \lambda}{1 -  \gamma^{d}} E(w_{opt})$. The significance of this result is that the rate of sensing can be made arbitrarily small (by increasing $d$), and yet convergence to $w_{opt}$ achieved (by taking $\lambda = 1$). The result might appear intriguing, since for fixed $\lambda < 1$, a tighter bound is obtained by increasing $d$ (making \textit{fewer} updates). Nonetheless, note that the bound is on the convergent limit; the best results after any \textit{finite} number of steps are likely to be obtained for some finite value of $d$. The bias-variance analysis of multi-step returns~\cite{Kearns+Singh:2000} applies as is to $d$: small values imply more bootstrapping and bias, large values imply higher variance. 

To demonstrate the effect of $d$ in practice, we construct a relatively simple MRP---described in Figure~\ref{fig:mrp}---in which linear TD$_{d}$($\lambda$) has to learn only a single parameter $w$. Figure~\ref{fig:mrperrors} shows the prediction errors after 100,000 steps (thus $\lfloor 100,000 / d \rfloor$ learning updates). When $\alpha$ and $\lambda$ are fixed, observe that the error  for smaller values of $\lambda$ is minimised at $d > 1$, suggesting that $d$ can be a useful parameter to tune in practice. However, the lowest errors can always be obtained by taking $\lambda$ sufficiently close to $1$ and suitably lowering $\alpha$, with no need to tune $d$. We obtain similar results by generalising ``True online TD($\lambda$)''~\cite{VanSeijen+MPMS:2016}; its near-identical plot is omitted.

\section{Control with Action-repetition}
\label{sec:controlwithaction-repetition}

In this section, we analyse frame-skipping in the control setting, wherein the agent is in charge of action selection. If sensing is restricted to every $d$-th step, recall from Section~\ref{sec:problemdefinition} that the agent must choose a $d$-length sequence of actions $b \in A^{d}$ at every sensing step. The most common approach~\cite{Braylan+HMM:2015,durugkar2016deep}
is to perform action-repetition: that is, to restrict
this choice to sequences of the \textit{same} action. This way the agent continues to have $|A|$ action sequences to consider (rather than $|A|^{d}$). It is also possible to consider $d$ as a parameter for the agent to itself learn, possibly as a function of state~\cite{lakshminarayanan2016dynamic,Lakshminarayanan+SR:2017}. We report some results from this approach in Section~\ref{sec:empiricalevaluation}, but proceed with our analysis by taking $d$ to be a fixed input parameter. Thus, the agent must pick an action sequence $b \in \{a^{d}, a \in A\}$.

It is not hard to see that interacting with input MDP $M = (S, A, R, T, \gamma)$ by repeating actions $d$ times is equivalent to interacting with an induced MDP $M_{d} = (S, A_{d}, R_{d}, T_{d}, \gamma^{d})$ without action-repetition~\cite{hansen1997reinforcement}. Here $A_{d} = \{a^{d}, a \in A\}$. For $s, s^{\prime} \in S, a \in A$, (1) let $R^{a}$ denote $R(\cdot, a)$ as an $|S|$-length vector---thus $R^{a}(s) = R(s, a)$---and (2) let $T^{a}$ denote $T(\cdot, a, \cdot)$ as an $|S| \times |S|$ matrix---thus $T^{a}(s, s^{\prime}) = T(s, a, s^{\prime}).$ Then $R_{d}(s, a) = R^{a}_{d}(s)$ and $T_{d}(s, a, s^{\prime}) = T^{a}_{d}(s, s^{\prime})$, where $$R^{a}_{d} = \sum_{j = 0}^{d - 1} (\gamma T^{a})^{j} R^{a}, \text{ and } T^{a}_{d} = (T^{a})^{d}.$$

\subsection{Price of inertia}
\label{subsec:priceofinertia}

The risk of using $d > 1$ in the control setting is that in some tasks, a single unwarranted repetition of action could be catastrophic. On the other hand,  in tasks with gradual changes of state, the agent must be able to recover. To quantify the amenability of task $M$ to action repetition, we define a term called its ``price of inertia'', denoted
$\Delta_{M}$. For $d \geq 1$, $s \in S, a \in A$, let $Q^{\star}_{M}(s, a^{d})$ denote the expected long-term reward of repeating action $a$ from state $s$ for $d$ time steps, and thereafter acting optimally on $M$. The price of inertia quantifies the cost of a \textit{single} repetition: $$\Delta_{M} \eqdef \max_{s \in S, a \in A} (Q^{\star}_{M}(s, a) - Q^{\star}_{M}(s, a^{2})).$$ $\Delta_{M}$ is a \textit{local}, horizon-independent property, which we expect to be small in many families of tasks. As a concrete illustration, consider the family of deterministic MDPs that have ``reversible'' actions. A calculation in Appendix~\ref{app:priceofinertiadeterministicmdps}~\footnote{Appendices are provided in the supplementary material.} shows that $\Delta_{M}$ for any such MDP $M$
is at most $4 R_{\max}$---which is a horizon-\textit{independent} upper bound.

To further aid our understanding of the price of inertia $\Delta_{M}$, we devise a ``pitted grid world'' task, shown in Figure~\ref{fig:gridworld-figure}. This task allows for us to control $\Delta_{M}$ and examine its effect on performance as $d$ is varied. The task is a typical grid world with cells, walls, and a goal state to be reached. Episodes begin from a start state chosen uniformly at random from a designated set. The agent can select ``up'', ``down'', ``left'', and  ``right'' as actions. A selected action is retained with probability 0.85, otherwise switched to one picked uniformly at random, and thereafter implemented to have the corresponding effect. There is a reward of $-1$ at each time step, except when reaching special ``pit'' states, which incur a larger penalty. It is precisely by controlling the pit penalty that we control $\Delta_{M}$. The task is undiscounted. Figure~\ref{fig:gridworld-figure} shows optimal policies for $d = 1$ and $d = 3$ (that is, on $M_{3}$); observe that they differ especially in the vicinity of pits (which are harder to avoid with $d = 3$).

%We shall revisit the pitted grid world as we proceed with our analysis.

\begin{figure}[t]
\centering
\subfloat[]{
  \includegraphics[width=0.98\linewidth]{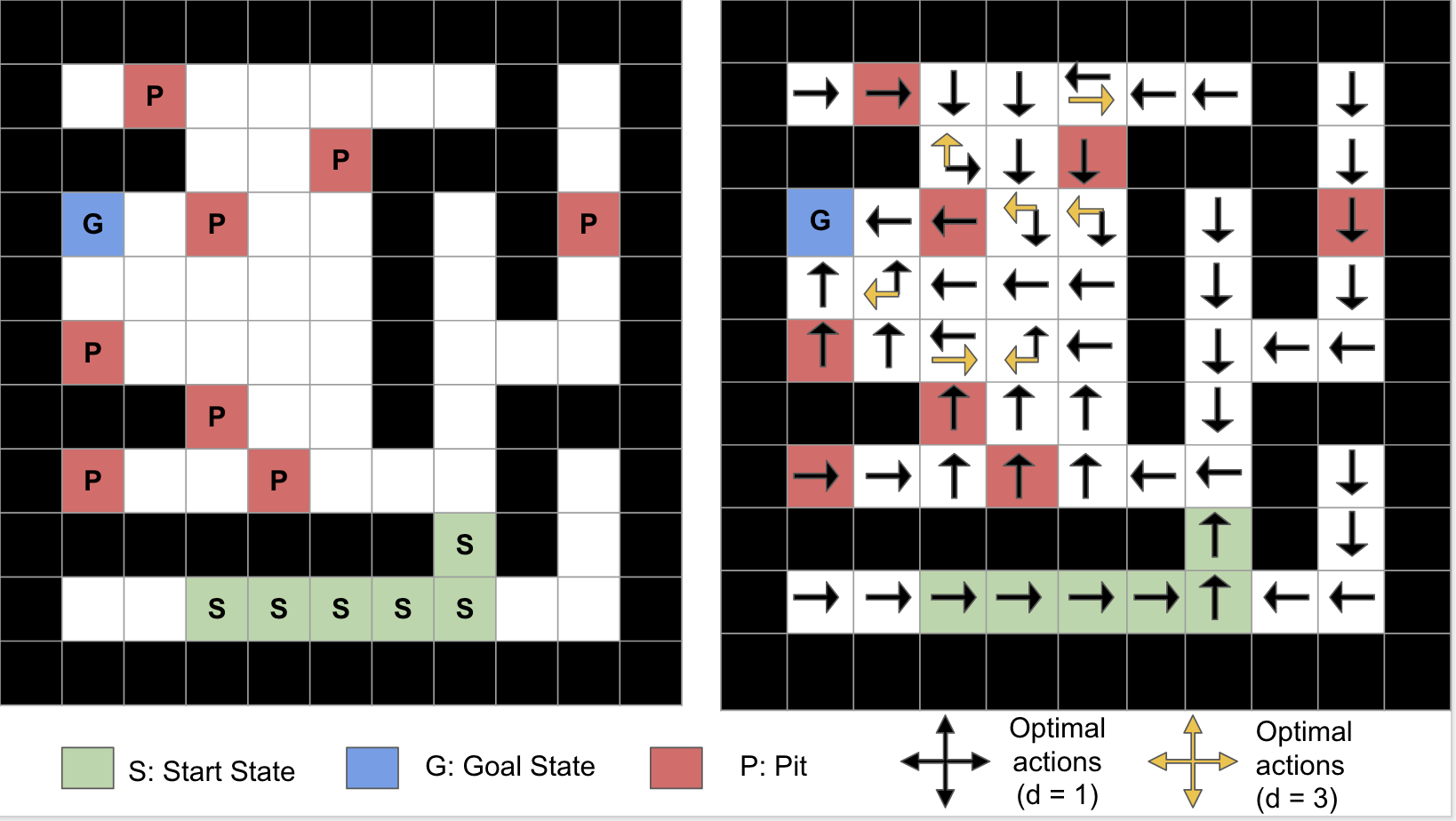}
\label{fig:gridworld-figure}
}

\mbox{
\hspace*{-0.2cm}\subfloat[]{
  \includegraphics[width=0.58\linewidth]{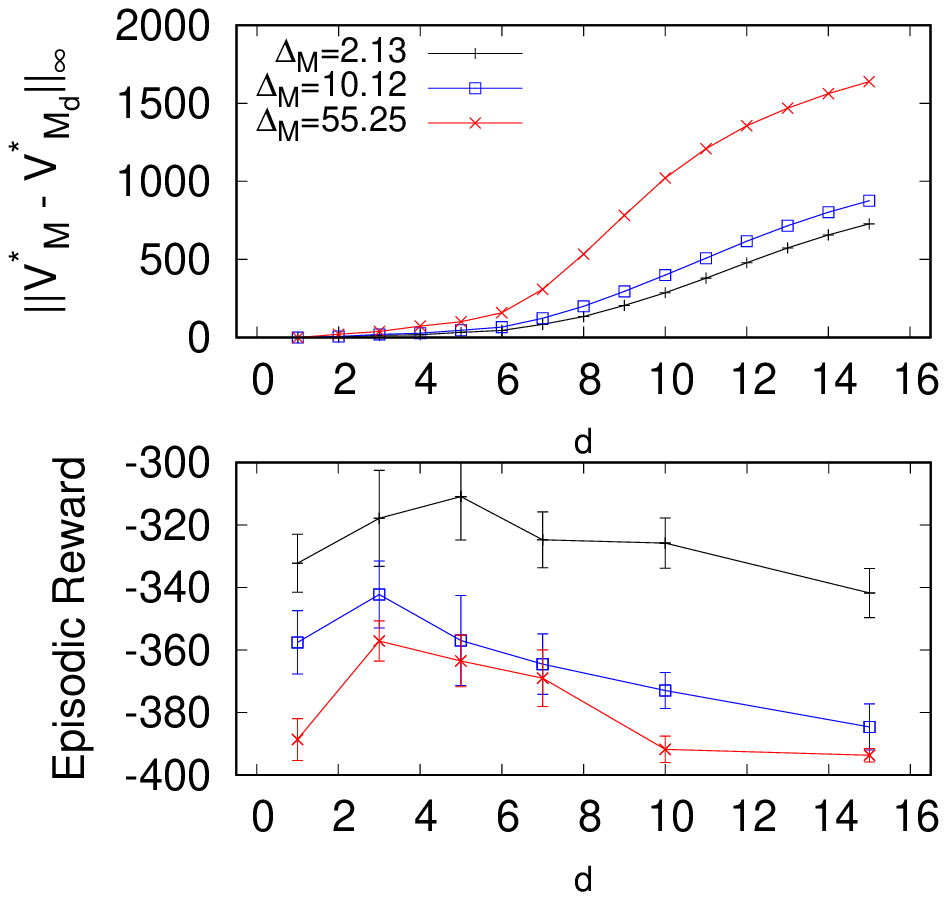}
\label{fig:diffvald}
}
\subfloat[]{
\hspace*{-1.0cm}\includegraphics[width=0.54\linewidth]{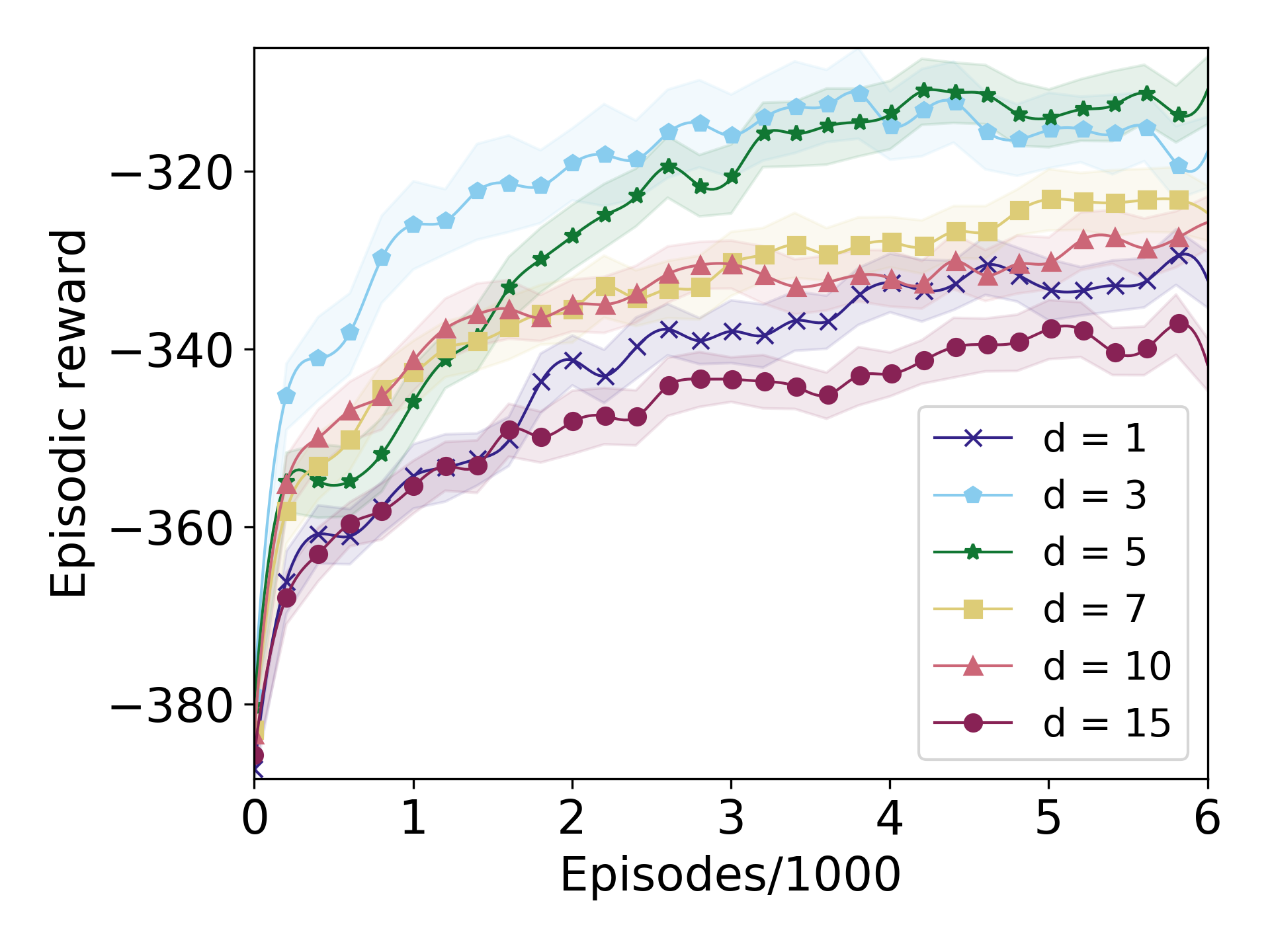}
  \label{fig:gammaqlearngridten}
}
}

\caption{(a) Left: the ``Pitted Grid World'' task (specified in text), using $\Delta_{M} = 10.12$. Right: Optimal policies when constrained to use frame-skip $d = 1$ (black/dark arrows) and $d = 3$ (golden/light) arrows (only one type is shown for states in which both policies give the same action). (b) The upper panel shows $||V^{\star}_{M} - V^{\star}_{M_{d}}||_{\infty}$, as a function of d for $\Delta_{M} \in \{2.13, 10.12, 55.26\}$. The lower panel shows the episodic reward obtained after 6000 episodes of Q-learning with aliased states, again as a function of $d$. (c) Full learning curves for Q-learning (for $\Delta_{M} = 2.13$), showing policy performance (evaluated separately for 50 episodes) at regular intervals. In this and subsequent experimental plots, error bars show one standard error, based on 5 or more independent runs.}
  \end{figure}

%For instance, Appendix~\ref{app:priceofinertia} shows that if $M$ has 

\subsection{Value deficit of action-repetition}
\label{subsec:valuedeficitofaction-repetition}

Naturally,  the constraint of having to repeat actions $d > 1$ times may limit the maximum possible long-term value attainable. We upper-bound the resulting deficit as a function of $\Delta_{M}$ and $d$. For MDP $M$, note that $V^{\star}_{M}$ is the optimal value function.\footnote{In our forthcoming analysis, we treat value and action value functions as vectors, with $\lVert\cdot\rVert_{\infty}$ denoting the max norm.}

\begin{lemma}
\label{lem:repetitionloss}
For $d \geq 1$, $\lVert V^{\star}_{M} - V^{\star}_{M_{d}} \rVert_{\infty} \leq \Delta_{M} \frac{1 - \gamma^{d - 1}}{(1 - \gamma)(1 - \gamma^{d})}$.
\end{lemma}
\begin{proof}
For $m \geq 2$ and $n \geq 1$, define the terms $G_{m} \eqdef \sum_{i = 0}^{m - 2} \gamma^{i}$
and $H_{m, n} \eqdef \sum_{i = 0}^{n - 1} \gamma^{mi}.$ First we prove
\begin{align}
\label{eqn:qstarbound}
Q^{\star}_{M}(s, a^{d}) \geq Q^{\star}_{M}(s, a) - \Delta_{M} G_{d}
\end{align}
for $s \in S, a \in A, d \geq 2$.  The result is trivial for $d = 2$. Assuming it is true for 
$d \leq m$, we get
\begin{align*}
Q^{\star}_{M}(s, a^{m + 1}) &= R(s, a) + \gamma \sum_{s^{\prime} \in S} T(s, a, s^{\prime}) Q^{\star}_{M}(s^{\prime}, a^{m}) \\   
&\geq R(s, a) + \gamma \sum_{s^{\prime} \in S} T(s, a, s^{\prime}) \{ Q^{\star}_{M}(s^{\prime}, a) - \Delta_{M} G_{m}  \} \\
& = Q^{\star}_{M}(s, a^{2}) - \gamma \Delta_{M} G_{m} \\
&\geq Q^{\star}_{M}(s, a) - \Delta_{M} - \gamma \Delta_{M} G_{m} = Q^{\star}_{M}(s, a) - \Delta_{M}G_{m + 1}.\\
\end{align*}
In effect, \eqref{eqn:qstarbound} bounds the loss from persisting action $a$ for $d$  
steps, which we incorporate in the long-term loss from action-repetition. To do so, we consider a policy $\pi: S \to \{a^{d}, a \in A\}$ that takes the same atomic actions as $\pi^{\star}_{M}$, but persists them for $d$ steps. In other words, for $s \in S$, $\pi(s) = a^{d} \iff  \pi^{\star}_{M}(s) = a$. For $j \geq 1$, let $U_{j}(s)$ denote the expected long-term reward accrued from state $s \in S$ by taking the first $j$ decisions based on $\pi$ (that is, applying $\pi$ for $jd$ time steps), and then acting optimally (with no action-repetition, according to $\pi^{\star}_{M}$). We prove by induction, for $s \in S$:
\begin{align}
\label{eqn:ejbound}
U_{j}(s) \geq V^{\star}_{M}(s) - \Delta_{M} G_{d} H_{d, j}.
\end{align}
For base case, we apply \eqref{eqn:qstarbound} and get
\begin{align*}
U_{1}(s) &= Q^{\star}_{M} (s, (\pi^{\star}_{M}(s))^{d})
\geq Q^{\star}_{M}(s, \pi^{\star}_{M}(s)) - \Delta_{M} G_{d}\\
&= V^{\star}_{M}(s) - \Delta_{M} G_{d} H_{d, 1}.
\end{align*}
Assuming the result true for $j$, and again using \eqref{eqn:qstarbound}, we establish it for $j + 1$.
\begin{align*}
&U_{j + 1}(s)\\
&= R_{d}(s, \pi^{\star}_{M}(s)) + \gamma^{d}
\sum_{s^{\prime} \in S} T_{d}(s, \pi^{\star}_{M}(s), s^{\prime})  U_{j}(s^{\prime}) \\ 
&\geq R_{d}(s, \pi^{\star}_{M}(s)) + \gamma^{d}
\sum_{s^{\prime} \in S} T_{d}(s, \pi^{\star}_{M}(s), s^{\prime})  
\{V^{\star}_{M}(s^{\prime}) - \Delta_{M} G_{d} H_{d,j}\}\\
&= Q^{\star}_{M}(s, (\pi^{\star}_{M}(s))^{d}) - \gamma^{d} \Delta_{M} G_{d} H_{d,j}\\
&\geq Q^{\star}_{M}(s, \pi^{\star}_{M}(s)) - \Delta_{M} G_{d} - \gamma^{d} \Delta_{M} G_{d} H_{d,j} \\
 &= V^{\star}_{M}(s) - \Delta_{M} G_{d} H_{d, j + 1}.
\end{align*}
Observe that $\lim_{j \to \infty} U_{j}(s) = V^{\pi}_{M_{d}}(s)$: the value of $s$ when $\pi$ is executed in $M_{d}$. The result follows by using $V^{\pi}_{M_{d}}(s) \leq V^{\star}_{M_{d}}(s)$, and substituting for $G_{d}$ and $H_{d, \infty}$.
\end{proof}
The upper bound in the lemma can be generalised to action value functions, and also shown to be tight. Proofs of the following results are given in appendices \ref{app:proofofq} and \ref{app:prooflowerbound}.

%bound in Lemma~\ref{lem:repetitionloss} is tight in the following sense; a proof by construction is given in Appendix~\ref{app:prooflowerbound}.

\begin{corollary}
\label{cor:repetitionlosswithq}
For $d \geq 1$, $\lVert Q^{\star}_{M} - Q^{\star}_{M_{d}} \rVert_{\infty} \leq \Delta_{M} \frac{1 - \gamma^{d - 1}}{(1 - \gamma)(1 - \gamma^{d})}$.
\end{corollary}

\begin{proposition}
\label{prop:lowerbound}
For every $\Delta > 0$, $d \geq 2$ , and $\gamma \in [0, 1)$, there exists an MDP $M$ with 
$\Delta_{M} = \Delta$ and discount factor $\gamma$ such that $\lVert V^{\star}_{M} - V^{\star}_{M_{d}} \rVert_{\infty} = \lVert Q^{\star}_{M} - Q^{\star}_{M_{d}} \rVert_{\infty} = \Delta_{M} \frac{1 - \gamma^{d - 1}}{(1 - \gamma)(1 - \gamma^{d})}$.
\end{proposition}

The matching lower bound in Proposition~\ref{prop:lowerbound} arises from a carefully-designed MDP; in practice we expect to encounter tasks $M$ for which the upper bound on $\lVert V^{\star}_{M} - V^{\star}_{M_{d}} \rVert_{\infty}$ is loose. Although our analysis is for infinite discounted reward, we expect $\Delta_{M}$ to play a similar role on undiscounted episodic tasks such as the pitted grid world. Figure~\ref{fig:diffvald} shows computed values of the performance drop from action-repetition, which monotonically increases with $d$ for every $\Delta_{M}$ value. Even so, the analysis to follow shows that using $d > 1$ might yet be the most effective if behaviour is \textit{learned}.

\subsection{Analysis of control with action-repetition}
\label{sec:analysisofcontrolwithaction-repetition}

We now proceed to our main result: that the deficit induced by $d$ can be offset by the benefit it brings in the form of a shorter task horizon. Since standard control algorithms such as Q-learning and Sarsa may not even converge with function approximation, we sidestep the actual process used to update weights. All we assume is that (1) the learning process produces as its output $\hat{Q}$, an approximate action value function, and (2) as is the common practice, the recommended policy $\hat{\pi}$ is greedy with respect to $\hat{Q}$: that is, for $s \in S$, $\hat{\pi}(s) = \argmax_{a \in A} \hat{Q}(s, a)$. We show that on an MDP $M$ for which $\Delta_{M}$ is small, it could in aggregate be beneficial to execute $\hat{\pi}$ with frame-skip $d > 1$; for clarity let us denote the resulting policy $\hat{\pi}_{d}: S \to \{a^{d}, a \in A\}$. The result holds regardless of whether $\hat{Q}$ was itself learned with or without frame-skipping, although in practice, we invariably find it more effective to use the same frame-skip parameter $d$ for both learning and evaluation. 

Singh and Yee~\cite{singh1994upper} provide a collection of upper bounds on the performance loss from acting greedily with respect to an \textit{approximate} value function or action value function. The lemma below is not explicitly derived in their analysis; we furnish an independent proof in Appendix~\ref{app:proofoflemmaapprox}.

\begin{lemma}
\label{lem:approximategreedyloss}
For MDP $M = (S, A, T, R, \gamma)$, let $\hat{Q}: S \times A \to \mathbb{R}$ be an $\epsilon$-approximation of $Q^{\star}_{M}$. In other words, $\lVert Q^{\star}_{M} - \hat{Q} \rVert_{\infty} \leq \epsilon$. Let $\hat{\pi}$ be greedy with respect to $\hat{Q}$. We have:
$\lVert V^{\star}_{M} - V^{\hat{\pi}}_{M} \rVert_{\infty} \leq \frac{2 \epsilon \gamma}{1 - \gamma}.$
\end{lemma}

The implication of the lemma is that the performance loss due to a prediction error scales as $\theta(\frac{\gamma}{1 - \gamma})$. Informally, $\frac{1}{1 - \gamma}$ may be viewed as the effective task horizon. Now observe that if a policy is implemented with frame-skip $d > 1$, the loss only scales as $\theta(\frac{\gamma^{d}}{1 - \gamma^{d}})$, which can be substantially smaller. However, the performance loss defined in Lemma~\ref{lem:approximategreedyloss} is with respect to optimal values in the \textit{underlying} MDP, which is $M_{d}$ (rather than $M$) when action-repetition is performed with $d >1$. Fortunately, we already have an upper bound on $\lVert V^{\star}_{M} - V^{\star}_{M_{d}} \rVert_{\infty}$ from Lemma~\ref{lem:repetitionloss}, which we can add to the one from Lemma~\ref{lem:approximategreedyloss} to meaningfully
compare $\hat{\pi}_{d}: S \to \{a^{d}, a \in A\}$ with $\pi^{\star}_{M}$. Doing so, we obtain our main result. 

\begin{theorem} 
\label{thm:aggregatebound}
Fix MDP $M = (S, A, R, T, \gamma)$, and $d \geq 1$. Assume that a learning algorithm  returns action-value function $\hat{Q}: S \times A \to \mathbb{R}$. Let $\hat{\pi}_{d}: S \to \{a^{d}, a \in A\}$ be greedy with respect to $\hat{Q}$. There exist constants $C_{1}(\gamma, d)$ and $C_{2}(M, \hat{Q})$ such that $$\lVert V^{\star}_{M} - V^{\hat{\pi}_{d}}_{M} \rVert_{\infty} \leq 
\Delta_{M} C_{1}(\gamma, d) + \frac{\gamma^{d}}{1 - \gamma^{d}} C_{2}(M, \hat{Q}),$$ with the dependencies of $C_{1}$ and $C_{2}$ shown explicitly in parentheses.
\end{theorem}

\begin{proof}
By the triangle inequality,
$$\lVert V^{\star}_{M} - V^{\hat{\pi_{d}}}_{M} \rVert_{\infty} \leq 
\lVert V^{\star}_{M} - V^{\star}_{M_{d}} \rVert_{\infty} +
\lVert V^{\star}_{M_{d}} - V^{\hat{\pi}_{d}}_{M} \rVert_{\infty}.$$ 

Lemma~\ref{lem:repetitionloss} upper-bounds the first RHS term by $\Delta_{M} C_{3}(\gamma, d)$, where $C_{3}(\gamma, d) = \frac{1 - \gamma^{d - 1}}{(1 - \gamma)(1 - \gamma^{d})}.$ Observe that the second RHS term may be written as $\lVert V^{\star}_{M_{d}} - V^{\hat{\pi}}_{M_{d}} \rVert_{\infty}$, which Lemma~\ref{lem:approximategreedyloss} upper-bounds by $\frac{2\epsilon\gamma^{d}}{1 - \gamma^{d}}$, where $\hat{Q}$ is an $\epsilon$-approximation of $Q^{\star}_{M_{d}}$. In turn, $\epsilon$ can be replaced by $\lVert Q^{\star}_{M_{d}} - \hat{Q}\rVert_{\infty}$, which is itself upper-bounded using the triangle inequality by $\lVert Q^{\star}_{M_{d}} - Q^{\star}_{M}\rVert_{\infty}
+ \lVert Q^{\star}_{M} - \hat{Q}\rVert_{\infty}$. Corollary~\ref{cor:repetitionlosswithq}
upper-bounds 
$\lVert Q^{\star}_{M} - Q^{\star}_{M_{d}} \Vert_{\infty}$ by $\Delta_{M} C_{3}(\gamma, d)$. As for $\lVert Q^{\star}_{M} - \hat{Q}\rVert_{\infty}$, observe that it only depends on $M$ and $\hat{Q}$. In aggregate, we have
\begin{align*}
\lVert V^{\star}_{M} - V^{\hat{\pi_{d}}}_{M} \rVert_{\infty}
&\leq 
\lVert V^{\star}_{M} - V^{\star}_{M_{d}} \rVert_{\infty} +
\lVert V^{\star}_{M_{d}} - V^{\hat{\pi}_{d}}_{M} \rVert_{\infty}\\
&\leq
\Delta_{M} C_{3}(\gamma, d) + \frac{2 \left(\Delta_{M}C_{3}(\gamma, d) + 
\lVert Q^{\star}_{M} - \hat{Q}\rVert_{\infty}
\right) \gamma^{d}}{1 - \gamma^{d}}\\
&=
\Delta_{M} C_{1}(\gamma, d) + \frac{\gamma^{d}}{1 - \gamma^{d}}C_{2}(M, \hat{Q})
\end{align*}
for appropriately defined $C_{1}(\gamma, d)$ and $C_{2}(M, \hat{Q})$.
\end{proof}
While the first term in the bound increases with $d$, the second term decreases on account of the shortening horizon. The overall bound is likely to be minimised by intermediate values of $d$ especially when the price of inertia ($\Delta_{M}$) is small and the approximation error 
($\lVert Q^{\star}_{M} - \hat{Q}\rVert_{\infty}$) large. We observe exactly this trend in the pitted grid world environment when we have an agent learn using Q-learning (with 0.05-greedy exploration and a geometrically-annealed learning rate). As a crude form of function approximation, we constrain (randomly chosen) pairs of neighbouring states to share  the same Q-values. Observe from figures \ref{fig:diffvald} (lower panel) and \ref{fig:gammaqlearngridten} that indeed the best results are achieved when $d > 1$.

%The bound in Theorem~\ref{thm:repetitionloss} does not directly apply to episodic tasks with no discounting. For such tasks, a ``back of the envelope'' calculation that takes $\gamma$ to be the termination probability at each step gives an expected horizon (steps before termination) of $T = 1 / (1 - \gamma)$ and an upper bound of $ \Delta_{M} T (1 - 1 / d)$ as the aggregate loss from AR.

\section{Empirical Evaluation}
\label{sec:empiricalevaluation}

The pitted grid world was an ideal task to validate our theoretical results, since it allowed us to control the price of inertia and to benchmark learned behaviour against optimal values. In this section, we evaluate action-repetition on more realistic tasks, wherein the evaluation is completely empirical. Our experiments test methodological variations and demonstrate the need for action-repetition for learning in a new, challenging task.

\subsection{Acrobot}
\label{subsec:acrobot}

We begin with Acrobot, the classic control task consisting of two links and two joints (shown in Figure~\ref{fig:acrobot-snapshot}). The goal is to move the tip of the lower link above a given height, in the shortest time possible. Three actions are available at each step: leftward,  rightward, and zero torque. Our experiments use the OpenAI Gym~\cite{brockman2016openai} implementation of Acrobot, which takes 5 actions per second. States are represented as a tuple of six features:  $\cos(\theta_1)$, $\sin(\theta_1)$, $\cos(\theta_2)$, $\sin(\theta_2)$, $\dot{\theta_1}$, and $\dot{\theta_2}$, where $\theta_1$ and $\theta_2$ are the link angles. The start state in every episode is set up around the stable point: $\theta_1$,
$\dot{\theta_1}$,
$\theta_2$, and
$\dot{\theta_2}$ are sampled uniformly at random from $[-0.1, 0.1]$.  A reward of -1 is given each time step, and $0$ at termination. Although Acrobot is episodic and undiscounted, we expect that as with the pitted grid world, the essence of Theorem~\ref{thm:aggregatebound} will still apply. Note that with control at 5 Hz, Acrobot episodes can last up to 500 steps when actions are selected uniformly at random. 
%on Acrobot take roughly 
%\textcolor{blue}{499, (agent can take a maximum of 500 actions per episode)}
%steps to terminate when actions are selected uniformly at random.

\setlength{\tabcolsep}{4pt}
\begin{figure}[b]
    \centering
    \mbox{
    \subfloat[]{
    \includegraphics[height=0.12\textheight]{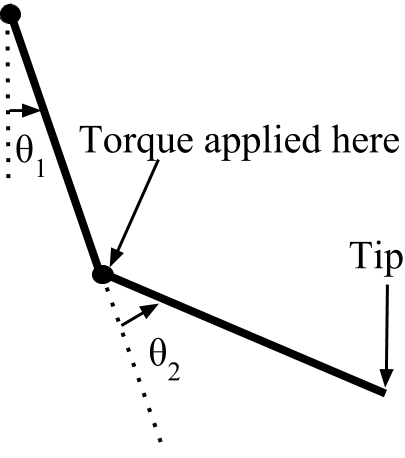}
    \label{fig:acrobot-snapshot}
    }
    \hspace{0.1cm}
    \subfloat[]{
    \includegraphics[height=0.15\textheight]{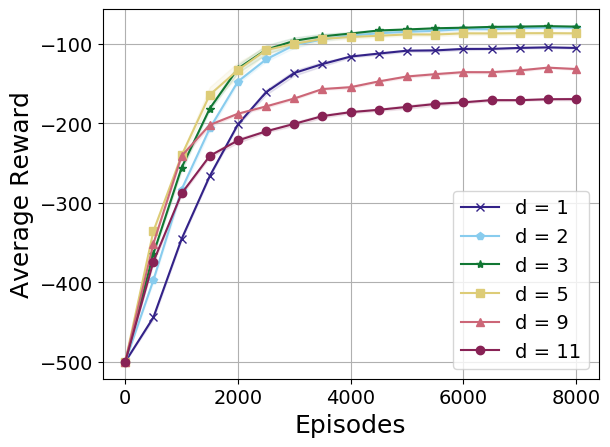}
    \label{fig:acrobot-sarsa}
    }
    }
    \hspace{0.1cm}
\raisebox{1.4cm}{    
    \subfloat[]{
\scriptsize
\begin{tabular}{
    |c|c|c|c|c|}
    \hline
$\gamma$ &$d = 1$ & $d = 2$ & $d = 3$ & $d = 5$\\ \hline
0.9 &$-123.1 (0.9)$ & $-346.6 (1.8)$ & $-310.5 (1.7)$ & $-270.4 (2.4)$\\
\hline
0.98&$-104.8 (0.5)$&$-79.7 (1.8)$&$-92.9 (6.2)$&$-104.6 (11.5)$ \\
\hline
0.99&$-\textbf{94.6} (3.6)$&$-\textbf{75.5} (0.7)$&$-74.1 (0.7)$&$-82.3 (1.2)$\\
\hline
0.999&$-95.8 (3.8)$&$-76 (0.8)$&$-\textbf{72.7} (0.8)$&$-81.4 (0.8)$\\
\hline
1&$-106.4 (3)$&$-76 (0.7)$&$-74.5 (1.2)$&$-\textbf{80} (0.2)$ \\
\hline
\end{tabular}
    \label{fig:acrobot-table}
\normalsize
}
    }
    
    \hspace*{-0.6cm}\mbox{
    \hspace*{0.4cm}\subfloat[]{
    \hspace*{-0.4cm}\includegraphics[height=0.15\textheight]{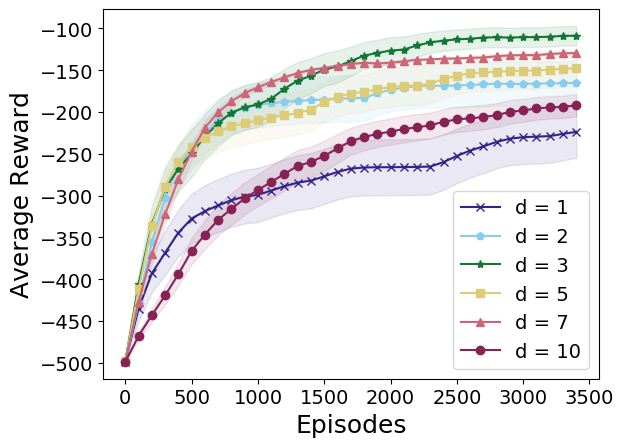}
    %\caption{Reinforce applied on Acrobot task with baseline. Shaded area represent Standard Error.}
    \label{fig:reinacroboterror}
    }
%    \subfloat[]{
%      \includegraphics[height=0.07\textheight]{archit-plots/covars_at_ep0.png}
        %  \caption{Gradient Co-variance Trace Plot for after 0 episodes training}
%      \label{fig:grad-covar-init}    
%    }
% \hspace{0.1cm}
    \hspace*{0.6cm}\subfloat[]{
      \hspace*{-0.6cm}\includegraphics[height=0.15\textheight]{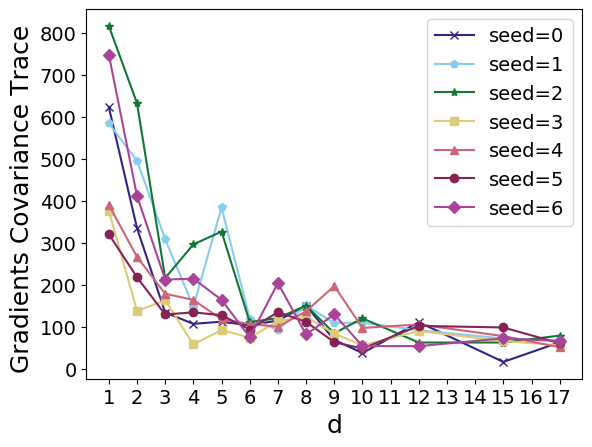}
        %  \caption{Gradient Co-variance Trace Plot for after 500 episodes training}
      \label{fig:grad-covar-med}    
    }

    }
\caption{(a) Screenshot of Acrobot. (b) Learning curves for Sarsa$_{d}$($\lambda$) with different frame-skip values $d$. (c) Episodic reward (and one standard error) obtained by Sarsa$_{d}$($\lambda$) after $8,000$ episodes of training with different $d$ and $\gamma$ combinations. (d) Learning curves for \textsc{reinforce} with different $d$ values. (e) For different $d$, an empirical estimate of the aggregate variance of $\nabla_{w}J(w)$ for $w$ found after $5,000$ episodes of training. For each seed, the policy found after 5,000 episodes of \textsc{reinforce} is frozen and run for $100$ transitions, each giving a sample gradient. The y axis shows the trace of the resulting covariance matrix.}
    \label{fig:acrobot}
\end{figure}

\begin{figure*}[t]
    \centering
    \mbox{
    \subfloat[]{
    \includegraphics[height=0.18\textheight]{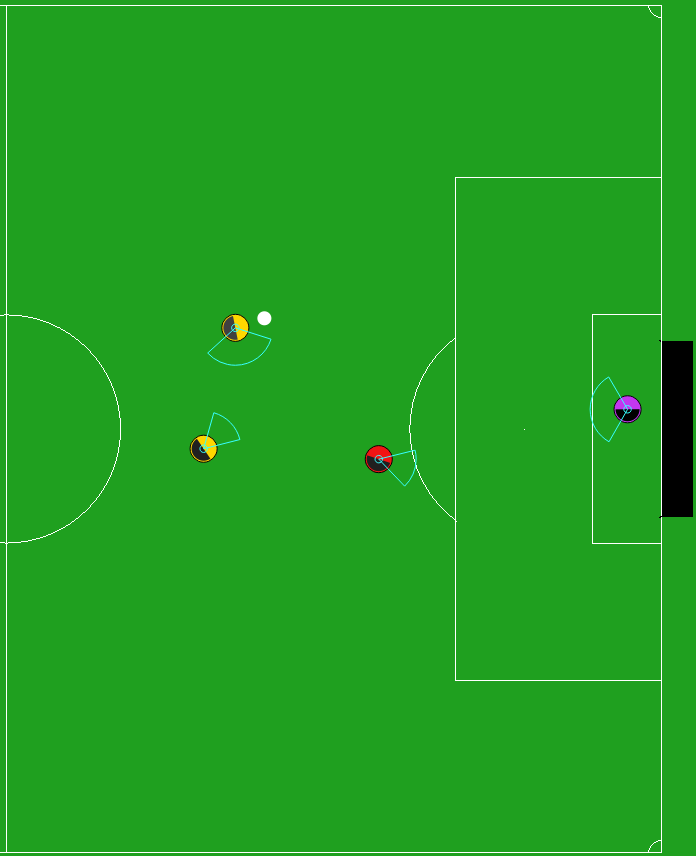}
    \label{fig:hfo-snapshot}
    }
     \hspace{0.8 cm}
    \subfloat[]{
    \includegraphics[height=0.2\textheight]{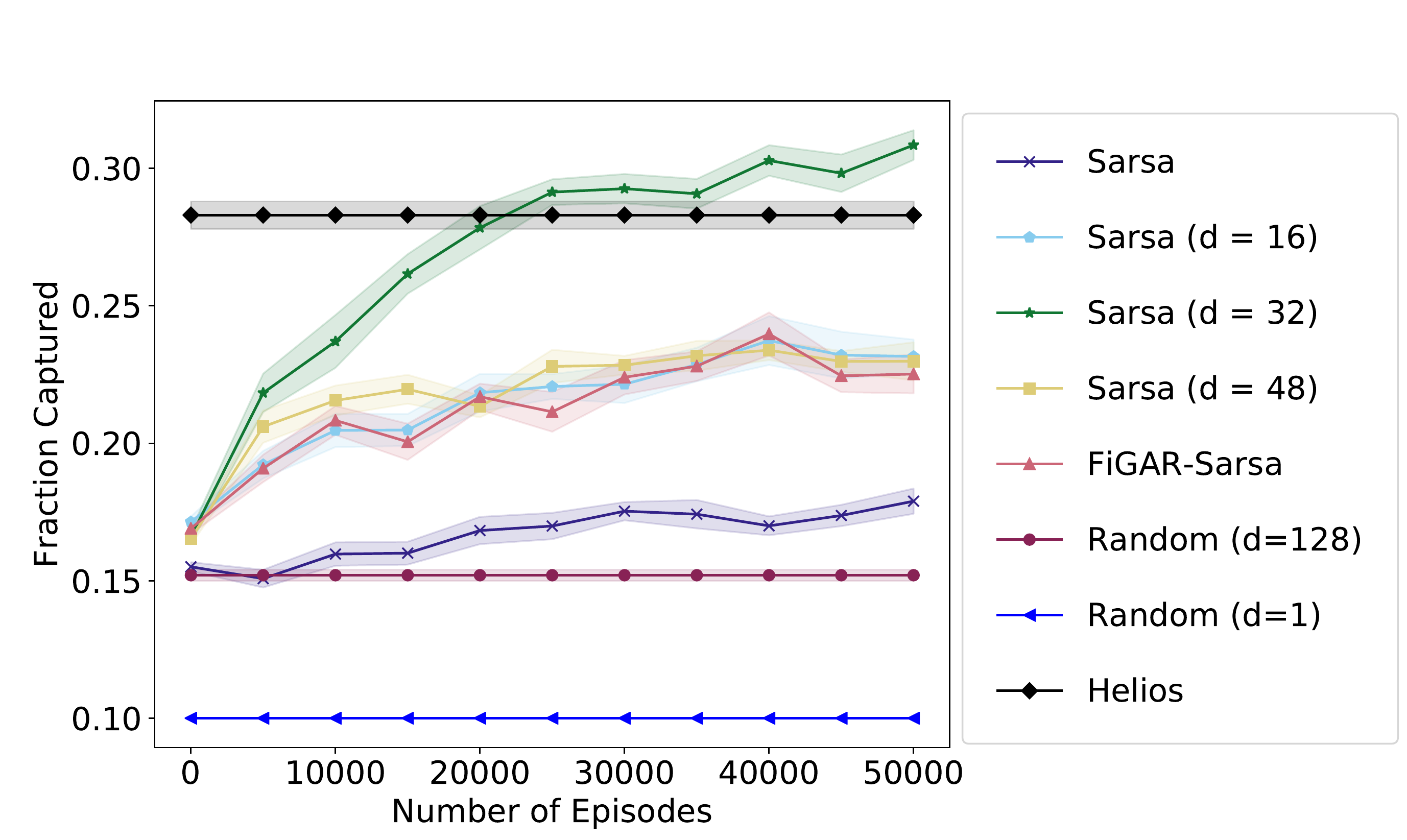}
    \label{fig:hfo-2v2}
    }
     \hspace{0.8 cm}
    \subfloat[]{
\includegraphics[height=0.18\textheight]{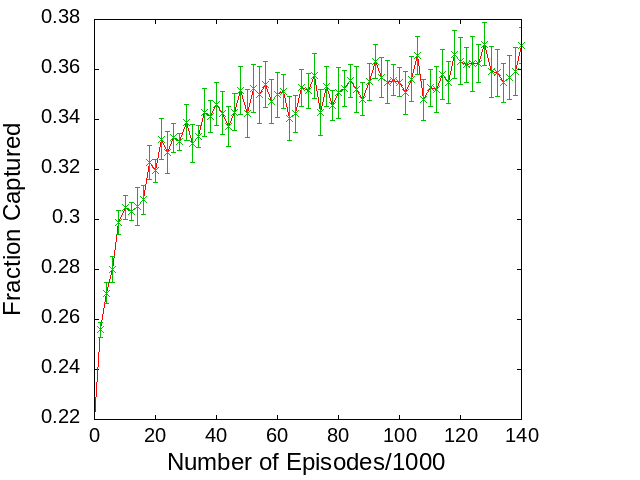}
\label{fig:hfo-exp3}
%\llap{\raisebox{3cm}{\includegraphics[width=0.2\textwidth]{plot_hist.png}} \\
\begin{picture}(0,0)
\put(-108,18){\includegraphics[width=0.15\textwidth]{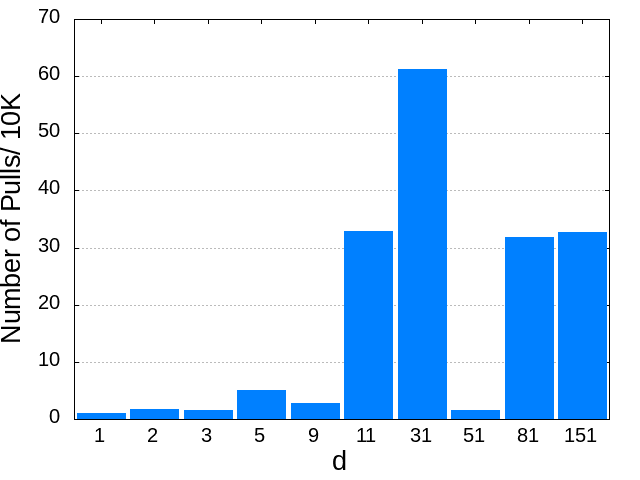}}
\end{picture}
    }
    }
    
% %    \subfloat[]{
% %    \includegraphics[height=0.2\textheight]{./figures/HFO_3v3.pdf}
% %    }
    
\caption{(a) Screenshot of 2v2 HFO. (b) Learning curves; ``Random'' and \textsc{helios} are static policies. (c) The learning curve for a ``meta'' learner that uses the EXP3.1 algorithm to switch between different values of $d$ (a single value is used within each episode). The inset shows the number of episodes (averaged from 10 independent runs) that each value of $d$ is invoked.}
    \label{fig:hfo}
\end{figure*}

We execute Sarsa$_{d}$($\lambda$), a straightforward generalisation of TD$_{d}$($\lambda$) to the control setting, using 1-dimensional tile coding~\cite[see Section 12.7]{Sutton+Barto:2018}.
Tuning other parameters to optimise results for $d = 1$, we set
$\lambda = 0.9$, 
%Optimising other parameters using a grid search
%We take $\lambda = 0.9$ (\textcolor{blue}{picked on what basis? : grid search on the set $\{0, .1, .2, .3, .4, .5, .6, .7, .8, .9, .95, .999\}$}),
%\textcolor{blue}{Did you do the grid search to optimise results for d = 1? : The grid search was done for all ds $\lambda=0.9$ gave the best results for d=1,2,3,5,9 and was the second best for d=11,31 }
$\alpha = 0.1$, and an initial exploration rate $\epsilon=0.1$, decayed by a factor of $0.999$ after each episode. Figure~\ref{fig:acrobot-sarsa} shows learning curves for different $d$ values. At 8,000 episodes, the best results are for $d = 3$; in fact Sarsa$_{d}$($\lambda$) with $d$ up to $5$ dominates Sarsa($\lambda$). It appears that Acrobot does not need control at 5 Hz; action-repetition shortens the task horizon and  enhances learning.\\

%When we compare policies learned by different algorithms after 8,000 episodes of learning, we find DI-Sarsa($\lambda, d$) with $\lambda=0.9,d = 3$ to be, by far, the most successful. The algorithm takes  $79.3 \pm 0.5$ steps to termination, much fewer than Sarsa($\lambda$) ($106.1 \pm 0.97$ with $\lambda = 0.9$), Sarsa($n$) ($114.5 \pm 2.5$ with $n=9$), and True Online Sarsa($\lambda$) ($105.1 \pm  1.61$ with $\lambda = 0.9$). Figure \ref{fig:acrobot} shows individual learning curves. Observe the initial improvement, and thereafter worsening, in performance as $d$ is increased, affirming the intuition we put forth in Section~\ref{sec:sardi}. Indeed the high frequency of action-selection in the simulator (5 Hz) leads to substiantial contiguity in actions picked by good policies, and so setting $d = 3$ does not preclude learning good policies. Happily, with only a third the number of updates to make, it greatly simplifies learning.

\noindent\textbf{Frame-skipping versus reducing discount factor.} If the key contribution of $d$ to successful learning is the reduction in horizon from $1/(1 - \gamma)$ to $1/(1 - \gamma^{d})$, a natural idea is to artificially reduce the task's discount factor $\gamma$, even without action-repetition. Indeed this approach has been found effective in conjunction with approximate value iteration~\cite{petrik2009biasing} and model-learning ~\cite{jiang2015dependence}. Figure~\ref{fig:acrobot-table} shows the values of policies learned by Sarsa$_{d}$($\lambda$) after $8,000$ episodes of training, when the discount factor $\gamma$ (originally $1$) is reduced. Other parameters are as before. As expected, some values of $\gamma < 1$ do improve learning. Setting $\gamma= 0.99$ helps the agent finish the task in $95.6$ steps: an improvement of $11.8$ steps over regular Sarsa($\lambda$). However, the configuration of $\gamma = 1, d = 5$ performs even better---implying that on this task, $d$ is more effective to tune than $\gamma$. Although decreasing $\gamma$ and increasing $d$ both have the effect of shrinking the horizon, the former has the consequence of revising the very definition of long-term reward. As apparent from Proposition~\ref{prop:valdistidentical}, $d$ entails no such change. That tuning these parameters in conjunction yields the best results (at $d = 3, \gamma = 0.999$) prompts future work to investigate their interaction. Interestingly, we find no benefit from using $\gamma < 1$ on the pitted grid world task.\\

\noindent\textbf{Action-repetition in policy gradient methods.} Noting that some of the recent successes of frame-skipping are on policy gradient methods~\cite{Sharma+LR:2017}, we run \textsc{reinforce}~\cite{Williams:1992} on Acrobot using action-repetition. Our controller computes an output for each action as a linear combination of the inputs, thereafter applying soft-max action-selection. The $21$ resulting weights $w$ (including biases) are updated by gradient ascent to optimise the episodic reward $J(w)$, using the Adam optimiser with initial learning rate $0.01$. We set $\gamma$ to $0.99$. Figure~\ref{fig:reinacroboterror} shows that yet again, performance is optimised at $d = 3$. Note that our implementation of \textsc{reinforce} performs baseline subtraction, which reduces the variance of the gradient estimate and improves results for $d = 1$. Even so, an empirical plot of the variance (Figure~\ref{fig:grad-covar-med}) shows that it falls further as $d$ is increased, with a relatively steep drop around $d = 3$. As yet, we do not have an analytical  explanation of this behaviour. Although known upper bounds on the variance of policy gradients~\cite{Zhao+HNS:2011} have a quadratic dependence on the task horizon, which is decreased by $d$ from $1/(1 - \gamma)$ to $1/(1 - \gamma^{d})$, they are also quadratic in the maximum reward, which is \textit{increased} by $d$ from $R_{\max}$ to $R_{\max} (1 + \gamma + \gamma^{2} + \dots + \gamma^{d - 1}) = R_{\max} (1 - \gamma^{d}) / (1 - \gamma)$ . We leave it to future work to explain the empirical observation of a significant reduction of the policy gradient variance with $d$ on Acrobot.

%We hypothesise that action-repetition reduces the variance of the policy gradient estimate, which is known to improve the convergence rate and the solution quality~\cite{Ng+Jordan:2000,peters2006policy}.

%We use a one layer neural network with input dimension being 6 and output dimension being 3 corresponding to the number of actions present to control the Acrobot. The following Gradient update rule is followed by collecting the rewards from the policy network. The learning rate is 0.01, using Adam as optimizer. $\gamma$ is set to $0.99$. Each step in the Acrobot gives -1 as reward except when the terminal state is reached in which case it is 0. We store the time averaged reward starting from -499 for each episode to judge the performance given by $RunningReward = RunningReward * 0.99 + 0.01 * (r)$ where r is the reward accumulated the agent till it reaches terminal state in an episode. We also experiment with the baseline version of the above algorithm in which.  The experiment is run for 32 distinct seeds. The reward is averaged and plotted. We also calculate the Standard Error given as $SE=\frac{\sigma}{\sqrt{n}}$ where $\sigma$ is the standard deviation and n is the number of seeds. The shaded region is $+SE$ and $-SE$ from the mean value. You can see the that the SE in case of DI=10 is much less compared to DI=1 which strengthens the argument that a higher DI leads to better learning. But the final performance peaks for an in-between DI due to strong constraint imposed on DI=10. In both with and without baseline the performance of DI=1 is worse than other DI.

\subsection{Action-repetition in new, complex domain} 
\label{subsec:action-repetitioninnewcomplexdomain}

Before wrapping up, we share our experience of implementing action-repetition in a new, relatively complex domain. We expect practitioners to confront similar design choices in other tasks, too.

The Half Field Offense (HFO) environment~\cite{ALA16-hausknecht} models a game of soccer in which an offense team aims to score against a defense team, when playing on one half of a soccer field (Figure \ref{fig:hfo-snapshot}). While previous investigations in this domain have predominantly focused on learning successful \textit{offense} behaviour, we address the task of learning \textit{defense}. Our rationale is that successful defense must anyway have extended sequences of actions such as approaching the ball and marking a player. Note that in 2 versus 2 (2v2) HFO, the average number of decisions made in an episode is roughly 8 for offense, and 100 for defense. We implement four high-level defense actions: \textsc{mark\_player}, \textsc{reduce\_angle\_to\_goal}, \textsc{go\_to\_ball}, and \textsc{defend\_goal}. The continuous state space is represented by $16$ features such as distances and positions~\cite{ALA16-hausknecht}. Episodes yield a single reward at the end: 1 for no goal and 0 for goal. No discounting is used. As before, we run Sarsa$_{d}$($\lambda$) with 1-dimensional tile coding. 

%Whereas on-line TD methods such as Sarsa have been very successful in learning offense (or similar) behaviour in HFO and in the related task of Keepaway~\cite{stone2000reinforcement}, we are unaware of any successful application to defense. For improving defense in a related 2D simulation setting, \cite{gabel2008case} \cite{gabel2008case} use batch learning, which tends to be more stable. \cite{kyrylov2007while} \cite{kyrylov2007while} propose an approach based on planning, rather than learning. For offense, the crux of decision making is when a player has possession of the ball. For defense, however, a player has to constantly evaluate and pick an appropriate action. 

%We ran experiments on the HFO environment \cite{ALA16-hausknecht} by training defense agents in 2v2 and 3v3 scenarios. 

In the 2v2 scenario, we train one defense agent, while using built-in agents for the goalkeeper and offense. Consistent with earlier studies~\cite{durugkar2016deep,McGovern+SF:1997}, we observe that action-repetition assists in exploration. With $d = 1$, random action-selection succeeds on only $10\%$ of the episodes; the success rate increases with $d$, reaching $15\%$ for $d = 128$. Figure~\ref{fig:hfo-2v2} shows learning curves: points are shown after every 5,000 training episodes, obtained by evaluating learned policies for 2,000 episodes. All algorithms use $\alpha = 0.1, \epsilon = 0.01, \lambda = 0.8$ (optimised for
Sarsa$_{1}$($\lambda$) at 50,000 episodes). Action-repetition shows a dramatic effect on Sarsa, which only registers a modest improvement over random behaviour with $d = 1$, but with $d = 32$, even outperforms a defender from the \textsc{helios} team
\cite{akiyama2018helios2018} that won the RoboCup competition in 2010 and 2012.\\ %Even if \textsc{helios} is not optimised for HFO, it is remarkable that a learning approach can surpass it. Clearly, AR plays a defining role in the success of Sarsa on this task.

\noindent\textbf{Optimising $d$.} A natural question arising from our observations is whether we can tune $d$ ``on-line'', based on the agent's experience. We obtain mixed results from investigating this possibility. In one approach, we augment the atomic set of actions with extended sequences;
%in another we allow the agent to learn separate $d$ values for different state-action pairs;
in
%yet
another we impose a penalty on the agent every time it switches actions. Neither of these approaches yields any appreciable benefit. The one technique that does show a rising learning curve, included in Figure~\ref{fig:hfo-2v2}, is FiGAR-Sarsa, under which we associate both action and $d$ (picked from $\{1, 2, 4, 8, 16, 32, 64\}$) with state, and update each $Q$-value independently. However, at 50,000 episodes of training, this method still trails Sarsa$_{d}$($\lambda$) with (static) $d = 32$ by a significant margin.

Observe that the methods described above all allow the agent to adapt $d$ \textit{within} each learning episode. On the other hand, the reported successes of tuning $d$ on Atari games~\cite{Lakshminarayanan+SR:2017,Sharma+LR:2017} are based on policy gradient methods, in which a fixed policy is executed in each episode (and updated between episodes). In line with this approach, we design an outer loop that treats each value of $d$ (from a finite set) as an \textit{arm} of a multi-armed bandit. A full episode, with Sarsa$_{d}$($\lambda$) updates using the corresponding, fixed frame-skip $d$ is played out on every pull. The state of each arm is saved between its pulls (but no data is shared between arms). Since we cannot make the standard ``stochastic'' assumption here, we use the EXP3.1 algorithm~\cite{auer2002nonstochastic}, which maximises expected payoff in the adversarial setting. Under EXP3.1, arms are sampled according to a probability distribution, which gets updated whenever an arm is sampled. Figure~\ref{fig:hfo-exp3} shows a learning curve corresponding to this meta-algorithm (based on a moving average of 500 episodes); we set $\lambda=0.9375$ for the best overall results. It is apparent from the curve and affirmed by the inset that Exp3.1 is quick to identify $d = 31$ as the best among the given choices ($d = 81$ and $d = 151$ are also picked many times due to their quick convergence, even if to suboptimal performance).

\section{Conclusion}
\label{sec:conclusion}

In this paper, we analyse frame-skipping a, simple approach that has recently shown much promise in applications of RL, and is especially relevant as technology continues to drive up frame rates and clock speeds. In the prediction setting, we establish that frame-skipping retains the consistency of prediction. In the control setting, we provide both theoretical and empirical justification for action-repetition, which applies the principle that tasks anyway having gradual changes of state can benefit from a shortening of the horizon. Indeed action-repetition allows TD learning to succeed on the defense variant of HFO, a hitherto less-studied aspect of the game. Although we are able to automatically tune the frame-skip parameter $d$ using an outer loop, it would be interesting to examine how the same can be achieved within each episode.\\

\bibliographystyle{unsrt}
\bibliography{references-1}

\newpage
\appendix

\section{Price of Inertia for Deterministic MDPs with Reversible Transitions}
\label{app:priceofinertiadeterministicmdps}
Consider a deterministic MDP $M$ in which transitions can be ``reversed'': in other words, for $s, s^{\prime} \in S, a \in A$, if taking $a$ from $s$ leads to $s^{\prime}$, then there exists an action $a^{c}$ such that taking 
$a^{c}$ from $s^{\prime}$ leads to $s$. Now suppose action $a$ carries the agent from $s$ to $s^{\prime}$, and thereafter from $s^{\prime}$ to $s^{\prime\prime}$. We have:
\begin{align*}
&Q^{\star}_{M}(s, a) - Q^{\star}_{M}(s, a^{2})\\
&= \gamma V^{\star}_{M}(s^{\prime}) - \gamma Q^{\star}_{M}(s^{\prime}, a)\\
&= \gamma V^{\star}_{M}(s^{\prime}) - \gamma \{R(s^{\prime}, a,) + \gamma V^{\star}_{M}(s^{\prime\prime})\}\\
&\leq \gamma V^{\star}_{M}(s^{\prime}) - \gamma \{R(s^{\prime}, a) + \gamma Q^{\star}_{M}(s^{\prime\prime}, a^{\text{c}})\}\\
&= \gamma V^{\star}_{M}(s^{\prime}) - \gamma \{R(s^{\prime}, a) + \gamma (R(s^{\prime\prime}, a^{c}) + \gamma V^{\star}_{M}(s^{\prime})) \}\\
&= \gamma V^{\star}_{M}(s^{\prime}) (1 - \gamma^{2}) - \gamma R(s^{\prime}, a) - \gamma^{2} R(s^{\prime\prime}, a^{c})\\
&\leq \gamma \frac{R_{\max}}{1 - \gamma} (1 - \gamma^{2}) + \gamma R_{\max} + \gamma^{2} R_{\max}\\
&= 2 \gamma (1 + \gamma) R_{\max}.
\end{align*}
Since $\gamma \leq 1$, it follows that $\Delta_{M}(s, a) = \max_{s \in S, a \in A} (Q^{\star}_{M}(s, a) - Q^{\star}_{M}(s, a^{2}))$ is at most $4R_{\max}$.

\section{Proof of Proposition~\ref{cor:repetitionlosswithq}}
\label{app:proofofq}

\begin{comment}
Q*_{Md}(s, a) = Rd(s, a) + gamma^{d} sum_{s’} Td(s, a, s’) V*_{Md}(s’)
>= Rd(s, a) + gamma^{d} sum_{s’} TD(s, a, s’) V*_{M}(s’) - gamma^{d} DeltaMGDHD
                       = Q*_{M}(s, a^{d}) - gamma^{d} DELTAMGDHD
                       >= Q*M(s, a) - DELTAMGDHD
                       = needed?

\end{comment}

The following bound holds for all $s \in S, a \in A$. The first ``$\geq$'' step follows from Lemma~\ref{lem:repetitionloss} and the second such step is based on an application of
\eqref{eqn:qstarbound}.

\begin{align*}
Q^{\star}_{M_{d}}(s, a)
&= R_{d}(s, a) + \gamma^{d} \sum_{s^{\prime} \in S} T_{d}(s, a, s^{\prime}) V^{\star}_{M_{d}}(s^{\prime})\\
&\geq 
R_{d}(s, a) + \gamma^{d} \sum_{s^{\prime} \in S} T_{d}(s, a, s^{\prime}) V^{\star}_{M}(s^{\prime}) - \gamma^{d} \Delta_{M} G_{d} H_{d, \infty}\\
&=
Q^{\star}_{M}(s, a^{d}) - \gamma^{d} \Delta_{M} G_{d} H_{d, \infty}\\
&\geq
Q^{\star}_{M}(s, a) - \Delta_{M}G_{d} -  \gamma^{d} \Delta_{M} G_{d} H_{d, \infty}\\
&=
Q^{\star}_{M}(s, a) - \Delta_{M} \frac{1 - \gamma^{d - 1}}{(1 - \gamma)(1 - \gamma^{d})}.
\end{align*}

\section{Proof of Proposition~\ref{prop:lowerbound}}
\label{app:prooflowerbound}

The figure below shows an MDP $M$ with states $1, 2, \dots, d$, and actions \textsc{stay} (dashed) and \textsc{move} (solid). All transitions are deterministic, and shown by arrows labeled with rewards. The positive reward $x$ is set to $\Delta / \gamma$.

\begin{figure}[h!]
\centering
\begin{tikzpicture}[node distance=2cm,auto,scale=0.6,transform shape,->]
\large
    \tikzset{state/.style={circle,draw=black,minimum size=10mm,thick}}
    \node[state] at (0, 0) (s1){$1$};
    \node[state] at (2, 0) (s2){$2$};
    \node[state] at (4, 0) (s3){$3$};
    \node[state] at (9, 0) (sd){$d$};
    \draw(s1) edge [thick, dashed, loop above] node {$0$}  (s1);
    \draw(s1) edge [thick] node[below] {$x$}  (s2);
    \draw(s2) edge [thick, dashed, loop above] node {$x$}  (s2);
    \draw(s2) edge [thick] node[below] {$0$}  (s3);
    \draw(s3) edge [thick, dashed, loop above] node {$x$}  (s3);
    \draw(s3) edge [thick] node[below] {$0$}  (5.5, 0);
    \node[] at (6.5,0) {$\dots$};
    \draw(7.5, 0) edge [thick] node[below] {$0$}  (sd);
    \draw(sd) edge [thick, dashed, loop above] node {$x$}  (sd);
    \draw(sd) edge [thick, bend left=30] node[below] {$0$}  (s1);
\normalsize
  \end{tikzpicture}
%\caption{MDP $M$ with states $1, 2, \dots, d$, and actions \textsc{stay} (dashed) and \textsc{move} (solid). All transitions are deterministic, and shown by arrows labeled with rewards. The positive reward $x$ is set to $\Delta / \gamma$. It can be verified that $\Delta_{M} = \Delta_{M}(1, \textsc{move}) = \Delta$, and also that $\lVert V^{\star}_{M} - V^{\star}_{M_{d}} \rVert_{\infty} = V^{\star}_{M}(1) - V^{\star}_{M_{d}}(1) = \Delta_{M} \frac{1 - \gamma^{d - 1}}{(1 - \gamma)(1 - \gamma^{d})}$.}
%\label{fig:lowerboundMDP}
\end{figure}

It can be verified that $\Delta_{M} = Q^{\star}_{M}(1, \text{\textsc{move}})
- Q^{\star}_{M}(1, \text{\textsc{move}}^{2})
 = \Delta$, and also that 
\begin{align*}
\lVert V^{\star}_{M} - V^{\star}_{M_{d}} \rVert_{\infty}
&= V^{\star}_{M}(1) - V^{\star}_{M_{d}}(1)\\
&= Q^{\star}_{M}(1, \textsc{1, move}) - Q^{\star}_{M_{d}}(1, \textsc{move})\\
&= \lVert Q^{\star}_{M} - Q^{\star}_{M_{d}} \rVert_{\infty}\\
&= 
\Delta_{M} \frac{1 - \gamma^{d - 1}}{(1 - \gamma)(1 - \gamma^{d})},
\end{align*} 
which matches the upper bound in Lemma~\ref{lem:repetitionloss}.

\section{Proof of Lemma~\ref{lem:approximategreedyloss}}
\label{app:proofoflemmaapprox}

We furnish the relatively simple proof below, while noting that many similar results (upper-bounds on the loss from greedy action selection) are provided by Singh and Yee~\cite{singh1994upper}.

For $s \in S$ and $j \geq 0$, let $U_{j}(s)$ denote the expected long-term discounted reward obtained by starting at $s$, following $\hat{\pi}$ for $j$ steps, and thereafter following an optimal policy $\pi^{\star}$. Our induction hypothesis is that $U_{j}(s) \geq V^{\star}_{M}(s) - 2 \epsilon \sum_{k = 1}^{j} \gamma^{k}$. As base case, it is clear that $U_{0}(s) = V^{\star}_{M}(s)$. Assuming the induction hypothesis to be true for $j$, we prove it for $j + 1$. We use the fact that $\hat{Q}$ is an $\epsilon$-approximation of $Q^{\star}_{M}$, and also that $\hat{\pi}$ is greedy with respect to $\hat{Q}$. For $s \in S$,

\begin{align*}
&U_{j + 1}(s)\\
&= R(s, \hat{\pi}(s)) + \gamma \sum_{s^{\prime}} T(s, \hat{\pi}(s), s^{\prime}) U_{j}(s^{\prime})\\
&\geq R(s, \hat{\pi}(s)) + \gamma \sum_{s^{\prime}} T(s, \hat{\pi}(s), s^{\prime}) V^{\star}_{M} (s^{\prime}) - 2 \epsilon \sum_{k = 1}^{j} \gamma^{k + 1}\\
&\geq R(s, \hat{\pi}(s)) + \gamma \sum_{s^{\prime}} T(s, \hat{\pi}(s), s^{\prime}) Q^{\star}_{M} (s^{\prime}, \hat{\pi}(s^{\prime})) - 2 \epsilon \sum_{k = 1}^{j} \gamma^{k + 1}\\
&\geq R(s, \hat{\pi}(s)) + \gamma \sum_{s^{\prime}} T(s, \hat{\pi}(s), s^{\prime}) \hat{Q}(s^{\prime}, \hat{\pi}(s^{\prime})) - \gamma \epsilon - 2 \epsilon \sum_{k = 1}^{j} \gamma^{k + 1}\\
&\geq R(s, \pi^{\star}(s)) + \gamma \sum_{s^{\prime}} T(s, \pi^{\star}(s), s^{\prime}) \hat{Q}(s^{\prime}, \pi^{\star}(s^{\prime})) - \gamma \epsilon - 2 \epsilon \sum_{k = 1}^{j} \gamma^{k + 1}\\
&\geq R(s, \pi^{\star}(s)) + \gamma \sum_{s^{\prime}} T(s, \pi^{\star}(s), s^{\prime}) Q^{\star}_{M}(s^{\prime}, \pi^{\star}(s^{\prime})) - 2 \gamma \epsilon - 2 \epsilon \sum_{k = 1}^{j} \gamma^{k + 1}\\
&= V^{\star}_{M}(s) - 2 \epsilon \sum_{k = 1}^{j + 1} \gamma^{k}.
\end{align*}
Since $\lim_{j \to \infty} U_{j}(s) = V^{\hat{\pi}}_{M}(s)$, we have $V^{\hat{\pi}}_{M}(s) \geq V^{\star}_{M}(s) - \frac{2 \epsilon \gamma}{1 - \gamma}$.

\end{document}